\documentclass[letterpaper]{article} % DO NOT CHANGE THIS
\usepackage{aaai25}  % DO NOT CHANGE THIS
\usepackage{times}  % DO NOT CHANGE THIS
\usepackage{helvet}  % DO NOT CHANGE THIS
\usepackage{courier}  % DO NOT CHANGE THIS
\usepackage[hyphens]{url}  % DO NOT CHANGE THIS
\usepackage{graphicx} % DO NOT CHANGE THIS
\usepackage{natbib}  % DO NOT CHANGE THIS AND DO NOT ADD ANY OPTIONS TO IT
\usepackage{caption} % DO NOT CHANGE THIS AND DO NOT ADD ANY OPTIONS TO IT
\usepackage{algorithm}
\usepackage{algpseudocode}
\usepackage{amsmath,amsfonts,bm,amssymb}
\usepackage{subcaption}
\usepackage{xspace}
\usepackage{amsthm}
\usepackage{makecell}

\DeclareMathOperator*{\argmax}{argmax}
\DeclareMathOperator*{\argmin}{argmin}
\newtheorem{theorem}{Theorem}
\newtheorem{lemma}{Lemma}
\newtheorem{assumption}{Assumption}
\newtheorem{proposition}{Proposition}
\newtheorem{definition}{Definition}
\newcommand{\psbo}{{BOIDS}\xspace}
\usepackage{newfloat}
\usepackage{listings}
\DeclareCaptionStyle{ruled}{labelfont=normalfont,labelsep=colon,strut=off} % DO NOT CHANGE THIS
\floatstyle{ruled}
\newfloat{listing}{tb}{lst}{}
\floatname{listing}{Listing}
\title{\psbo: High-Dimensional Bayesian Optimization via \\ Incumbent-Guided Direction Lines and Subspace Embeddings}
\author{
    Lam Ngo\textsuperscript{\rm 1},
    Huong Ha\textsuperscript{\rm 1},
    Jeffrey Chan\textsuperscript{\rm 1},
    Hongyu Zhang\textsuperscript{\rm 2}
}
\affiliations{
    \textsuperscript{\rm 1}The Royal Melbourne Institute of Technology (RMIT), Australia\\
    \textsuperscript{\rm 2}Chongqing University, China\\
    s3962378@student.rmit.edu.au, 
    huong.ha@rmit.edu.au,
    jeffrey.chan@rmit.edu.au,
    hyzhang@cqu.edu.cn
}

\usepackage{bibentry}
\begin{document}

\maketitle

\begin{abstract}
When it comes to expensive black-box optimization problems, Bayesian Optimization (BO) is a well-known and powerful solution. Many real-world applications involve a large number of dimensions, hence scaling BO to high dimension is of much interest. However, state-of-the-art high-dimensional BO methods still suffer from the curse of dimensionality, highlighting the need for further improvements. In this work, we introduce \psbo, a novel high-dimensional BO algorithm that guides optimization by a sequence of one-dimensional direction lines using a novel tailored line-based optimization procedure. To improve the efficiency, we also propose an adaptive selection technique to identify most optimal lines for each round of line-based optimization. Additionally, we incorporate a subspace embedding technique for better scaling to high-dimensional spaces. We further provide theoretical analysis of our proposed method to analyze its convergence property. Our extensive experimental results show that \psbo outperforms state-of-the-art baselines on various synthetic and real-world benchmark problems.
\end{abstract}

% Uncomment the following to link to your code, datasets, an extended version or similar.
%
\begin{links}
    \link{Code}{https://github.com/LamNgo1/boids}
\end{links}

\section{Introduction}

Bayesian Optimization (BO) \cite{shahriari2015taking, garnett_bayesoptbook_2023} is a powerful optimization method, capable of solving expensive black-box objective functions. BO has applications in many fields, notably hyperparameter tuning for machine learning models \cite{ turner2021bayesian, Wang2024BO_HPO}, neural architecture search \cite{ru2021interpretableNAS}, reinforcement learning \cite{parker2022autoRL}, engineering \cite{shields2021BO_chemical, ament2023BO_engineer}, robotics \cite{mayr2022BO_robotics}.

BO is an iterative method that can optimize expensive, black-box objective functions in a sample-efficient manner. In each iteration, BO constructs a surrogate model from observed data to approximate the objective function. To propose the next data points for observation, BO optimizes an acquisition function derived from the surrogate model and an optimization policy. The acquisition function quantifies the potential information gain from new observations, guiding the search towards solutions that maximize this gain. Common acquisition functions are computationally inexpensive to optimize and may even have tractable solutions under certain conditions. As a result, BO problems are transformed into less costly optimization tasks. 

As recent applications involve optimizing problems with hundreds to thousands of dimensions, scaling BO to such high dimensions has become a prevalent research topic \cite{binois2022BOsurvey, wang2023BOsurvey}. However, BO's performance significantly degrades as the dimension increases, partly due to the over-exploration of the acquisition function \cite{Eriksson2019TuRBO, ngo2024cmabo}. Classical approaches for optimizing BO acquisition functions, such as gradient solvers or evolutionary algorithms, operate over the entire search space, which can suffer from the curse of dimensionality. In high dimensions, the exponentially large search space makes optimizing the acquisition function computationally expensive (if the entire space is considered) or less accurate (if approximation techniques are used). Among existing solutions, LineBO \cite{kirschner2019linebo} introduces line-based optimization, which restricts the acquisition function optimization to a one-dimensional line (\textit{guiding line}), enhancing efficiency. This approach from LineBO can be interpreted as moving the data points following predefined lines as guiding directions. 
Despite its strong theoretical foundation, LineBO is not empirically competitive against recent high-dimensional BO baselines \cite{Letham2020Alebo, bardou2024dumbo}.

In this work, we propose a novel high-dimensional BO method inspired by the line-based optimization approach, where the optimization process is guided by our proposed novel guiding lines learnt from observed data. Specifically, inspired by Particle Swarm Optimization (PSO) algorithm \cite{kennedy1995pso}, we propose to use incumbents - the best solutions found so far - to determine the directions of these guiding lines. We make use of two types of incumbents: (1) the global incumbent across all data points and (2) the personal incumbent of each data point throughout its history. By using these incumbents, we encourage exploitation in regions closer to the incumbents, mitigating the over-explorative behavior in high dimensions. We term our proposed guiding lines as \textit{incumbent-guided} lines. We then introduce a new line-based optimization process tailored for these incumbent-guided lines, using a novel multi-objective acquisition function.
Furthermore, to increase the sample efficiency, we propose to select the best line in each iteration using a multi-armed bandit technique. Finally, we also propose to incorporate a subspace embedding technique \cite{papenmeier2022baxus} to further improve the overall performance. We name our method \textit{High-dimensional \underline{B}ayesian \underline{O}ptimization via \underline{I}ncumbent-guided \underline{D}irection Lines and \underline{S}ubspace Embeddings} (\psbo). To provide deeper insight into the novel incumbent-guided lines, we provide theoretical analysis on the convergence property and derive simple regret bound for \psbo. Our experimental results demonstrate that \psbo outperforms state-of-the-art high-dimensional BO methods on a collection of synthetic and real-world benchmark problems.
We summarize our contributions as follows.
\begin{itemize}
    \item We propose \psbo, a novel high-dimensional BO algorithm that incorporates incumbent-guided line directions into a tailored line-based optimization process, employs a multi-armed bandit technique to select the optimal directions, and uses subspace embedding technique to boost the overall performance.
    \item We provide theoretical analysis on the convergence property of our proposed method.
    \item We demonstrate that \psbo empirically outperforms state-of-the-art baselines on a comprehensive set of synthetic and real-world benchmark problems.
\end{itemize}

\section{Background}
\subsection{Bayesian Optimization}
Without loss of generality, let us consider the minimization problem: given an \textit{expensive black-box} objective function $f: \mathcal{X} \rightarrow \mathbb{R}$ where $\mathcal{X} \subset \mathbb{R}^d$ is the search space, the goal is to find the global optimum $\mathbf{x}^*$ of the objective function $f$,
\begin{equation} \label{eq:problem}
    \mathbf{x}^* \in \arg{\min_{\mathbf{x}\in\mathcal{X}}{f(\mathbf{x})}},
\end{equation}
using the least number of function evaluations.

BO iteratively proposes potential data point $\mathbf{x}$ and obtain noisy function evaluation $y=f(\mathbf{x}) + \varepsilon$, where $\varepsilon \sim \mathcal{N}(0,\sigma^2)$ represents observation noise. In each iteration, BO constructs a surrogate model from the observed dataset $\mathcal{D}=\{\mathbf{x}^{(i)}, y^{(i)}\}|_{i=1}^{t}$ up to iteration $t$, then constructs an acquisition function to decide next data point for observation. 

The surrogate model for BO is a statistical model that approximates the objective function. The most common type of BO surrogate model is Gaussian Process (GP), which is characterized by a mean function for the prior belief and a kernel that captures the behavior of the objective function, such as smoothness \cite{williams2006GP}. Other common types of BO surrogate model include TPE \cite{bergstra2011TPE}, Random Forest \cite{Frank2011SMAC}, neural networks \cite{muller2023pfns4bo}. 

The acquisition function $\alpha: \mathcal{X} \rightarrow \mathbb{R}$ quantifies the potential information gain at individual point $\mathbf{x}$ across the search space. By maximizing the acquisition function $\mathbf{x}^{(t+1)}\in \argmax_{\mathbf{x}\in\mathcal{X}}{\alpha(\mathbf{x})}$, which maximizes the information gain, the next data point for observation is selected.
Common choices of acquisition functions include 
% Probability of Improvement \cite{Kushner1964PI}, 
Expected Improvement \cite{Mockus1978EI}, Upper Confidence Bound \cite{srinivas2009gpucb}, Thompson Sampling \cite{Thompson1933TS}. Despite their different motivations and properties, acquisition functions must strike a balance between exploitation - sampling in regions where the function values are expected to be optimal, and exploration - sampling in regions where there is significant uncertainty about the objective function.

\subsection{Line-based Bayesian Optimization}
LineBO \cite{kirschner2019linebo} addresses high-dimensional BO problems by performing BO on a series of one-dimensional lines (\textit{guiding lines}). In each iteration $t$, LineBO defines a guiding line $\mathcal{L}=\mathcal{L}(\mathbf{\hat{x}}, \mathbf{v}) \in \mathcal{X}$ that passes through the current observed data point $\mathbf{\hat{x}} = \mathbf{x}^{(t)}$ and follows a uniformly random direction $\mathbf{v}$, i.e., $\mathbf{v}$ is uniformly sampled from a $d$-dimensional unit sphere. Then LineBO finds the next data point for evaluation by maximizing the acquisition function over the guiding line: $\mathbf{x}^{(t+1)} = \argmax_{\mathbf{x}\in\mathcal{L}}{\alpha(\mathbf{x})}$. This line-based optimization helps optimize the acquisition function more efficiently, especially in high dimensions, where the search space $\mathcal{X}$ is exponentially large. Despite having strong theoretical property, the empirical performance of LineBO is not competitive with recent state-of-the-art high-dimensional BO methods. We hypothesize that this is due to (1) the uniformly random construction of $\mathbf{v}$ and (2) the restriction of acquisition optimization on 1D lines which significantly reduces the number of potential data points. 
Inspired by theoretical and empirical foundations of LineBO, in this paper, we propose a novel algorithm for line-based optimization that guides optimization by direction $\mathbf{v}$ learnt from observed data.

\subsection{Subspace Embedding} \label{sec-background:subspace_embedding}
Subspace embedding methods address the high-dimensional optimization problems by leveraging latent subspaces for optimization. In this work, we focus on subspace embedding via linear random transformation \cite{wang2016rembo, Nayebi2019hesbo, Letham2020Alebo}. The main idea is to assume the existence of a low-dimensional subspace (active subspace) $\mathcal{Z}$ with dimension $d_e \leq d$, a function $g:\mathcal{Z} \rightarrow \mathbb{R}$ and a projection matrix $\mathbf{T}:\mathcal{X} \rightarrow \mathcal{Z}$ such that $\forall \mathbf{x}\in \mathcal{X}, \ g(\mathbf{Tx})=f(\mathbf{x})$. Then subspace embedding reduces the high-dimensional optimization problem to a lower-dimension subspace, effectively mitigate the curse of dimensionality issue. As the effective dimension $d_e$ is unknown in practice, subspace embedding methods instead operate in a low-dimensional target subspace $\mathcal{A}$ with a predefined target dimension $d_\mathcal{A}$. In subspace embedding BO methods, each iteration proposes a data point $\mathbf{x}_\mathcal{A} \in \mathcal{A}$, then uses a projection matrix $\mathbf{S}: \mathcal{A} \rightarrow \mathcal{X}$ for function evaluation. The projection matrix is defined as a sparse matrix $\mathbf{S}\in\{0, \pm 1\}^{d_\mathcal{A} \times d}$ \cite{Nayebi2019hesbo}, therefore, the function evaluation is $f(\mathbf{S}^\intercal \mathbf{x}_\mathcal{A})$.

\section{Related Work}
Various prior research have addressed the high-dimensional BO problems and can be categorized as follows. 

\paragraph{Effective Surrogate Models.} This approach replaces the common BO surrogate model, Gaussian Process \cite{williams2006GP}, with other models that scale better to high-dimensional spaces, including Tree-Parzen Estimator \cite{bergstra2011TPE}, Random Forest \cite{Frank2011SMAC} and neural networks \cite{snoek2015GP_DNN, springenberg2016GP_BNN}. Recently, Prior-data Fitted Networks (PFNs) \cite{muller2021pfns} was developed to mimic GP and Bayesian neural network to approximate the posterior predictive distribution for BO \cite{muller2023pfns4bo}.

\paragraph{Subspace Embedding.} This approach transforms the high-dimensional optimization problem into a low-dimensional problem on a target subspace. LineBO \cite{kirschner2019linebo} operates on one-dimensional target subspaces. HESBO \cite{Nayebi2019hesbo} and ALEBO \cite{Letham2020Alebo} can handle target subspaces with higher dimensions, yet the target subspaces' dimensions need to be predefined as a hyperparameter, which may lead to suboptimal performance. To overcome this, BAxUS \cite{papenmeier2022baxus} adaptively increases the target subspace's dimension, which has been shown to be more effective. 

\paragraph{Variable Selection.} This approach learns a subset of variables to perform low-dimensional BO. DropOut \cite{li2018dropout} is a classical approach that randomly selects the variables. SAASBO \cite{eriksson21saasbo} learns the most relevant low-dimensional subspace using Sparse Axis-Aligned Subspace Prior for GP modelling.

\paragraph{Search Space Partitioning.} This approach partitions the search spaces to smaller local regions, to enhance the surrogate modelling and encourage exploitation of potential regions that might contain the global optimum. TuRBO \cite{Eriksson2019TuRBO} defines the local regions as hyper-rectangles, updating their sizes based on the success or failure of the optimization process. As TuRBO's local region are heuristic, other works attempt to learn the local regions from data. LAMCTS \cite{Wang2020LAMCTS} learns potential local regions with non-linear boundary, improving the performance at the cost of being more computational expensive. CMA-BO \cite{ngo2024cmabo} leverages the Covariance Matrix Adaptation strategy to learn hyper-ellipsoid local regions, which are updated by inexpensive computation. 

\paragraph{Objective Function Decomposition.} This approach assumes an additive structure of the objective functions and constructs multiple low-dimensional GPs to approximate the high-dimensional objective functions \cite{kandasamy2015AddGPUCB, hoang2018DECHBO, han2021treeBO}. Recently, RDUCB \cite{ziomek2023rducb} suggests that random decomposition without learning from data can result in competitive performance. DumBO \cite{bardou2024dumbo} relaxes the assumption on the maximum dimension of each decomposition, which was required to be fixed and low in previous works.

\paragraph{Evolutionary Algorithms (EA).} EA is a family of black-box optimization algorithms known for their empirical success in high-dimensional optimization problems. CMA-ES \cite{Hansen2001CMAES} is a well-known powerful black-box optimization method. Particle Swarm Optimization (PSO) \cite{kennedy1995pso} makes use of the incumbents to guide the swarm (set of data points) for searching the optima, which serves as the inspiration for our proposed method. Empirically, EA methods are generally less sample-efficient than BO algorithms.

\section{\psbo: High-Dimensional Bayesian Optimization via Incumbent-guided Direction Lines and Subspace Embeddings}
In this section, we present our proposed method \psbo. 
As illustrated in Fig. \ref{fig:illustration}, \psbo maintains a set of $m$ incumbent-guided lines simultaneously (Sec. \ref{sec-method:incumbent-line}). 
Then, the best line among the set is selected using a multi-armed bandit strategy (Sec. \ref{sec-method:line-direction-selection}). 
Subsequently, we sample the next data point using our novel line-based optimization and multi-objective acquisition functions (Sec. \ref{sec-method:line-based-optimization}). Finally, our subspace embedding technique improves the efficiency of \psbo (Sec. \ref{sec-method:subspace-embedding}). An overview of \psbo is illustrated in Fig. \ref{fig:illustration}.

\begin{figure*} [t]
  \centering
  \includegraphics[width=\textwidth]{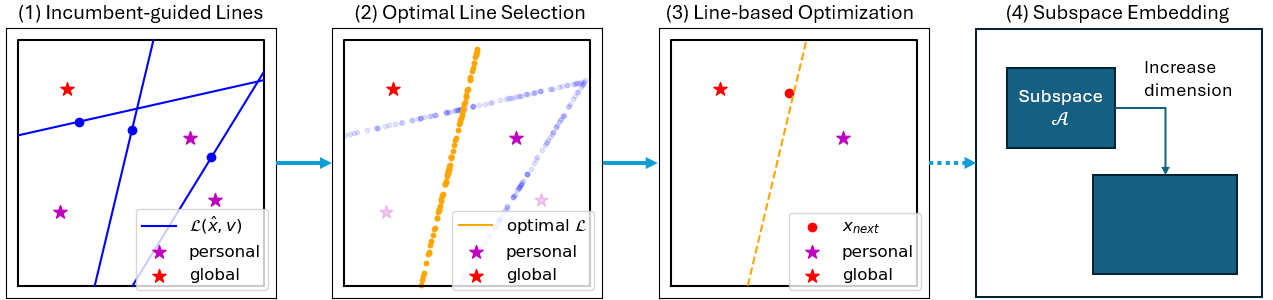}
\caption{Illustration of our proposed \psbo algorithm. (1) The incumbent-guided lines are constructed from the personal and global incumbents. (2) The optimal lines are selected based on Thompson Sampling MAB strategy. (3) Line-based Optimization is performed following the optimal line. (4) Subspace embedding technique is incorporated to enhance the performance.} \label{fig:illustration}
\end{figure*}

\subsection{Incumbent-guided Direction Lines} \label{sec-method:incumbent-line}
Our method, \psbo, is a line-based optimization approach in which we propose to compute the guiding lines using the \textit{incumbents}, i.e., the best data points found so far. \psbo inspires the idea from the PSO algorithm \cite{kennedy1995pso}, particularly the velocity formula, to define the incumbent-guided lines. At iteration $t$, let us denote $\{\mathbf{x}^{(t)}_i\}|_{i=1}^{m}$ as the set of $m$ data points (aka the $m$ particles in PSO).
Let us denote $\mathbf{p}^{(t)}_i$ as the personal incumbent (aka PSO personal best) of the $i$-th data point $\mathbf{x}^{(t)}_i$, which is its best historical location (in terms of function values) up to iteration $t$. Let us denote $\mathbf{g}^{(t)}$ as the global incumbent (aka PSO global best), which is the best data point in observed dataset $\mathcal{D}$ up to $t$. Details on these incumbents can be found in Appendix Sec. \ref{sec-appendix:incumbents}.
The incumbent-guided direction $\mathbf{v}^{(t+1)}_i$ for the $i$-th data point $\mathbf{x}^{(t)}_i$ in the next iteration is defined as,
\begin{equation} \label{eq:velocity}
    \mathbf{v}^{(t+1)}_i = w\bar{\mathbf{x}}^{(t)}_i + \mathbf{r}_1 c_1 \bar{\mathbf{p}}_i^{(t)} + \mathbf{r}_2 c_2 \bar{\mathbf{g}}^{(t)},
\end{equation}
where $\bar{\mathbf{x}}^{(t)}_i$ is the displacement vector of the $i$-th point $\mathbf{x}^{(t)}_i$ relative to its last update, $\bar{\mathbf{p}}_i^{(t)}=\mathbf{p}^{(t)}_i - \mathbf{x}^{(t)}_i$ is the personal best direction of the $i$-th point $\mathbf{x}^{(t)}_i$, and $\bar{\mathbf{g}}^{(t)}=\mathbf{\mathbf{g}}^{(t)} - \mathbf{x}^{(t)}_i$ is the global best direction. The uniformly random vectors $\mathbf{r}_1$ and $\mathbf{r}_2$ are sampled from $\mathcal{U}([0,1]^d)$, and $w$, $c_1$, $c_2$ are coefficients that control exploration and exploitation. 
The set of incumbent-guided lines for the $m$ data points are then defined as $\mathbb{L}=\{\mathbf{\mathcal{L}(\hat{x}}_i, \mathbf{v}_i)\}|_{i=1}^{m}$, where each line $\mathcal{L}_i=\mathcal{L}(\mathbf{\hat{x}}_i, \mathbf{v}_i)$ passes through a point $\mathbf{\hat{x}}_i = \mathbf{x}^{(t)}_i$ and follows directions $\mathbf{v}_i$ as in Eq. (\ref{eq:velocity}). These incumbent-guided lines encourage exploitation in promising regions by focusing the search on incumbents' directions, while preserving the explorative behavior due to inherent randomness. As shown by the performance of PSO \cite{kennedy1995pso}, directions $\mathbf{v}_i$ can point towards areas containing the global optimum, effectively guiding the line-based optimization approach towards it. Note the settings for incumbent-guided lines including $m$, $w$, $c_1$, $c_2$ can follow PSO (See Sec. \ref{sec-exp:baselines_settings}).

\subsection{Adaptive Line Selection} \label{sec-method:line-direction-selection}
In this section, we describe our proposed technique that adaptively selects the most optimal direction from the set of $m$ incumbent-guided line directions $\mathbb{L}=\{\mathcal{L}_i\}|_{i=1}^{m}$ formulated in Sec. \ref{sec-method:incumbent-line}. We formulate this problem as a multi-armed bandit (MAB) problem, in which each of $m$ lines $\mathcal{L}_i$ represents an arm to be pulled, and the goal is to choose the most optimal arm (i.e., the arm with the highest reward/performance). To solve this MAB problem, we employ Thompson Sampling (TS) \cite{Thompson1933TS}, a well-known algorithm for decision-making tasks \cite{agrawal2012TSMAB, Eriksson2019TuRBO}. 
We leverage the GP surrogate model, trained from all observed data so far, as the probability distribution to compute the reward of each arm.
To compute the reward, we first sample a random realization $g$ from the posterior of the GP surrogate model. Then we sample a pool of candidate data points across all $m$ lines and compute the TS values based on the random realization $g$. The reward $r_i$ of line $\mathcal{L}_i$ is defined as the maximum TS values among all data points on $\mathcal{L}_i$, i.e., $r_i=\max_{\mathbf{x}\in\mathcal{L}_i}{g(\mathbf{x})}$. Finally, we maximize the reward among $m$ lines to identify the optimal line chosen for line-based optimization. 
Formally, the line $\mathcal{L}_{i^*}$ is selected by maximizing the following equation,
\begin{equation} \label{eq:line-select}
    i^* = \argmax_{i=1,\dots,m}{\max_{\mathbf{x}\in \mathcal{L}_i}{g(\mathbf{x})}} \quad \text{where} \quad g \sim \mathcal{GP}(\mathcal{D}),
\end{equation}
and $\mathcal{GP}(\mathcal{D})$ denotes the GP posterior trained from all the observed data $\mathcal{D}$ so far.
Note that even though we use TS for line direction selection, the line-based optimization process can employ different acquisition functions. 
The pseudo code for this procedure is shown in Alg. \ref{alg:line-select} in Appendix Sec. \ref{sec-appendix:line-direction-selection}.

\subsection{Incumbent-guided Line-based Optimization} \label{sec-method:line-based-optimization}
In this section, we present the line-based optimization process developed for the proposed incumbent-guided lines. Let the optimally chosen incumbent-guided direction line from Sec. \ref{sec-method:line-direction-selection} be denoted as $\mathcal{L}_{i^*}=\mathcal{L}(\hat{\mathbf{x}}_{i^*}^{(t)}, \mathbf{v}_{i^*})$.
The main idea is to suggest the next data point for observation by guiding $\hat{\mathbf{x}}_{i^*}^{(t)}$ following $\mathcal{L}_{i^*}$, while being quantified by a common acquisition function $\alpha(.)$, such as EI, TS.
Instead of optimizing $\alpha(.)$ solely over the guiding line $\mathcal{L}_{i^*}$ (as done in LineBO), we optimize $\alpha(.)$ in the entire search space $\mathcal X$ while imposing additional constraints that incorporate the incumbent information. 
Specifically, we define the Euclidean distance constraints to the personal incumbent of $\hat{\mathbf{x}}_{i^*}^{(t)}$ and the current global incumbents: $L_{\mathbf{p}}(\mathbf{x}) = \Vert\mathbf{x}-  \mathbf{p}_{i^*}^{(t)}\Vert$ and $L_{\mathbf{g}}(\mathbf{x}) = \Vert\mathbf{x} - \mathbf{g}^{(t)}\Vert$, respectively. These constraints encourage exploitation in the direction of incumbents, which can guide the search towards the global optimum. We then formulate the acquisition optimization problem as a multi-objective (MO) optimization problem where the Euclidean distances are treated as additional objectives alongside the acquisition function.
The MO acquisition function for \psbo is then,
\begin{equation} \label{eq:moo-acq}
    \mathcal{P} = \argmax_{\mathbf{x}\in\mathcal{X}}{ \big( f_\alpha(\mathbf{x}), f_{\mathbf{p}}(\mathbf{x}), f_{\mathbf{g}}(\mathbf{x}) \big)},
\end{equation}
where $f_\alpha(.) = \alpha(.)$, $f_\mathbf{p}(.)=-L_\mathbf{p}(.)$ and $f_\mathbf{g}(.)=-L_\mathbf{g}(.)$.
In practice, to direct the acquisition optimization to follow the defined incumbent-guided line $\mathcal{L}$, we initialize Eq. (\ref{eq:moo-acq}) with data points randomly sampled on $\mathcal{L}$, then optimize Eq. (\ref{eq:moo-acq}) using off-the-shelf MO solvers, such as evolutionary algorithms. Solving Eq. (\ref{eq:moo-acq}) yields a Pareto set of $n_\mathcal{P}$ solutions $\mathcal{P} = \{\mathbf{x}_i\}|_{i=1}^{n_p}$ and a Pareto front $\mathcal{F} = \{f_{\alpha}(\mathbf{x}_i), f_{\mathbf{p}}(\mathbf{x}_i), f_{\mathbf{g}}(\mathbf{x}_i)\}|_{i=1}^{n_p}$.
% , where $f_{\alpha}(.), f_{\mathbf{p}}(.), f_{\mathbf{g}}(.)$ are the objective values corresponding to objectives $\alpha(.)$, $-L_{\mathbf{p}}(.)$ and $-L_{\mathbf{g}}(.)$ in Eq. (\ref{eq:moo-acq}) respectively. 
Finally we select the next BO data points for observation by maximizing $\mathcal F$ w.r.t. $f_{\alpha}(.)$, 
\begin{equation} \label{eq:moo-acq-select}
    \mathbf{x}^* = \argmax_{\mathbf{x}\in\mathcal{P}}{f_{\alpha}}(\mathbf{x}).
\end{equation}
Eq. (\ref{eq:moo-acq-select}) selects the solution with the best acquisition objective on the Pareto front, ensuring that the selected data points are optimal (based on acquisition function values), while also being non-dominated by other solutions in terms of the Euclidean distances to the incumbents (based on MO optimization). See Alg. \ref{alg:line-based-opt} in Appendix Sec. \ref{sec-appendix:line-based-optimization} for the pseudo code.

\subsection{Subspace Embedding} \label{sec-method:subspace-embedding}
To further enhance \psbo's capability in handling high-dimensional optimization problems, we employ a linear random subspace embedding strategy. Specifically, we leverage the adaptive expanding subspace embedding, known as the BAxUS embedding \cite{papenmeier2022baxus}, and execute \psbo through a series of low-dimensional subspace with varying dimensionalities. Compared to other random linear embedding rules, such as HESBO \cite{Nayebi2019hesbo}, the BAxUS embedding produces the highest worst-case probability to contain the global optimum, given the same conditions (see further discussion in Sec. \ref{sec-theory:global}).
In detail, given a low-dimensional subspace $\mathcal{A}$ with dimension $d_\mathcal{A}$, \psbo proposes data points $\mathbf{x}_\mathcal{A} \in \mathcal{A}$, which are projected to the original input space $\mathcal{X}$ via a projection matrix $\mathbf{S}$.
The optimization on $\mathcal{A}$ is performed within a budget $T_{d_\mathcal{A}}$, after which the subspace $\mathcal{A}$ is expanded to a subspace with higher dimensionality. 
During the expansion of $\mathcal{A}$, previous observations and the set of $m$ incumbent-guided lines are preserved via embedding matrix $\mathbf{S}$. 
Note that we only employ the BAxUS embedding rule, not other components of the 
 BAxUS algorithm such as trust regions. Specifically, we adopt the rule for increasing the subspace dimensionality per expansion and the number of iterations budget $T_{d_\mathcal{A}}$ for each subspace $\mathcal{A}$. The subspace expansion process is conducted repeatedly until $d_\mathcal{A}$ reaches the input dimension $d$, at which \psbo operates in the input space $\mathcal{X}$.

\subsection{The Proposed \psbo Algorithm} \label{sec-method:overall-algorithm}
%In summary, 
Overall, the \psbo algorithm operates as follows (pseudo-code in Alg. \ref{alg:ps-bo-alg} and illustration in Fig. \ref{fig:illustration}). At the beginning, based on a predefined initial subspace dimension $d_\mathcal{A}$, we initialize the embedding matrix $\mathbf{S}$, compute the budget $T_{d_\mathcal{A}}$, randomly initialize observed dataset $\mathcal{D}_0 \in \mathcal{A}$ and randomly choose $m$ data points in $\mathcal{D}_0$ (lines \ref{alg-line:init_subspace}-\ref{alg-line:init_data_points}). 
Then, in each iteration, after constructing a GP surrogate model $\mathcal{GP}$, we sequentially compute the set of $m$ incumbent-guided directions $\mathbb{L}=\{\mathcal{L}(\mathbf{\hat{x}}_i, \mathbf{v}_i)\}|_{i=1}^{m}$, select the most optimal line direction $\mathcal{L}_{i^*}$ and perform line-based optimization to suggest the next data point for observation $\mathbf{x}_\mathcal{A}\in \mathcal{A}$ (lines \ref{alg-line:build_gp}-\ref{alg-line:line_optimize}). 
Then, we project $\mathbf{x}_\mathcal{A}$ to $\mathcal{X}$ via projection matrix $\mathbf{S}$ and compute objective function evaluation (lines \ref{alg-line:project_data}-\ref{alg-line:compute_fx}). 
The process repeats until the budget $T_{d_\mathcal{A}}$ for subspace $\mathcal{A}$ is depleted. After that, if $d_\mathcal{A} < d$, we expand the subspace $\mathcal{A}$ dimensionality, update the matrix $\mathbf{S}$ and recompute the budget $T_{d_\mathcal{A}}$ (line \ref{alg-line:update_subspace}). Otherwise, we discard all previous observed data, and restart optimization in the input subspace (line \ref{alg-line:restart}). The algorithm repeats until the predefined budget $T$ is depleted, then returns the best solution found $\mathbf{x}^*$ as the final solution. 

\begin{algorithm}[tb] 
   \caption{The \psbo Algorithm}
   \label{alg:ps-bo-alg}
\begin{algorithmic}[1]
   \State {\bfseries Input:} Function $f(.)$, budget $T$, initial dimension $d_\mathcal{A}$, number of data points $m$
   \State {\bfseries Output:} The optimum $\mathbf{x}^*$
   \State Initialize subspace $\mathcal{A}$, matrix $\mathbf{S}$, budget $T_{d_\mathcal{A}}$ \label{alg-line:init_subspace}
   \State Initialize $\mathcal{D}_0 \in \mathcal{A}$ and randomly choose $m$ data points \label{alg-line:init_data_points}
   \While{budget $T$ not depleted}
     \While{budget $T_{d_\mathcal{A}}$ not depleted}
       \State Build GP surrogate model $\mathcal{GP}_\mathcal{A}$ from $\mathcal{D}$ \label{alg-line:build_gp} 
       \State Compute set of lines $\mathbb{L}=\{\mathcal{L}(\mathbf{\hat{x}}_i, \mathbf{v}_i)\}|_{i=1}^{m}$ \label{alg-line:compute_lines}
       \State $\mathcal{L}_{i^*} \leftarrow \Call{LINE-SELECT}{\mathcal{GP}_\mathcal{A}, \mathbb{L}}$ \label{alg-line:line_select} \Comment{Sec. \ref{sec-method:line-direction-selection}}
       \State $\mathbf{x}_\mathcal{A} \leftarrow \Call{LINE-OPT}{\mathcal{GP}_\mathcal{A}, \mathcal{L}_{i^*}}$ \label{alg-line:line_optimize} \Comment{Sec. \ref{sec-method:line-based-optimization}}
       \State Project $\mathbf{x}_\mathcal{X} \leftarrow \mathbf{S}^\intercal \mathbf{x}_\mathcal{A}$ \label{alg-line:project_data} 
       \State Compute function evaluation $y=f(\mathbf{x}_\mathcal{X}) +\varepsilon$ \label{alg-line:compute_fx}
       \State Update $\mathcal{D} \leftarrow \{\mathbf{x}_\mathcal{A}, \mathbf{x}_\mathcal{X}, y\}$
     \EndWhile
     \If{$d_\mathcal{A} < d$}
       \State Increase $d_\mathcal{A}$, update $\mathbf{S}$, compute $T_{d_\mathcal{A}}$ \label{alg-line:update_subspace}
     \Else
       \State Discard previous dataset, set $d_\mathcal{A} = d$, $\mathbf{S}=\mathbb{I}^d$ \label{alg-line:restart}
     \EndIf
   \EndWhile
   \State Return $\mathbf{x}^* = \argmin_{\mathbf{x}^{(i)} \in \mathcal{D}} \{y^{(i)}\}_{i=1}^N$
\end{algorithmic}
\end{algorithm}

\section{Theoretical Analysis}
In this section, we provide theoretical analysis of the convergence property of \psbo, including local and global convergence properties. We prove that our incumbent-guided line-based optimization has a sub-linear simple regret bound, enabling \psbo to converge to (local) optimum. Moreover, as we incorporate a subspace embedding technique, we discuss the global convergence property given the chosen embedding rule. 

\subsection{Local Convergence} \label{sec-theory:local}
The main idea is to show that \psbo, leveraging the proposed incumbent-guided directions, can converge to a (local) optimum, supported by a sub-linear simple regret bound. The simple regret $r_t$ is defined as the difference between the global optimum value $f^*$ and the best function value $y_i$ found so far until iteration $t$. For a minimization problem, $r_t=\min{y^{(i)}}|_{i=1}^t - f^*$. We will now proceed to derive the bound for the simple regret.
We first impose the following assumption on the objective function $f$.
\begin{assumption} [RKHS] \label{assumption:rkhs}
The objective function $f$ is a member of a reproducing kernel Hilbert space $\mathcal{H}(k)$ with known kernel $k: \mathcal{X} \times \mathcal{X} \rightarrow \mathbb{R}$ and bounded norm ${\Vert f \Vert}_{\mathcal{H}} \le B$.
\end{assumption}
We leverage the standard analysis of BO derived in \cite{srinivas2009gpucb} for our regret analysis. In particular, the regret bound is often expressed in terms of the \textit{maximum information gain} $\gamma_T$, whose bounds are known for common kernels such as Mat\'{e}rn and RBF \cite{srinivas2009gpucb, seeger2008information}. Similar to LineBO \cite{kirschner2019linebo}, we impose the following assumptions on $\gamma_T$.
\begin{assumption} [Maximum Information Gain] \label{assumption:gamma_t}
Let $k: \mathcal{X} \times \mathcal{X} \rightarrow \mathbb{R}^+$ be a one-dimensional kernel. Let $\kappa = 2/(2\nu + 2)$, where $\nu$ is the smoothness factor of kernel $k$. The maximum information gain is then
\begin{equation}
    \gamma_T(k) \le \mathcal{O}\left(T^\kappa \log{T}\right).
\end{equation}
\end{assumption}
In \psbo, as we use the incumbent-guided line as the search direction, we derive the following lemma.
\begin{lemma}[Incumbent-guided Search Direction] \label{lemma:informed-direction}
Let $\mathbf{r}_1, \mathbf{r}_2\in \mathbb{R}^d$ be two random vectors sampled from a uniform distribution $\mathcal{U}([0,1]^d)$, $\mathbf{h}_1, \mathbf{h}_2 \in \mathbb{R}^d$ are arbitrary fixed vectors, and the element-wise multiplications $\mathbf{v}_1 = \mathbf{r}_1 \circ \mathbf{h}_1$ and $\mathbf{v}_2 = \mathbf{r}_2 \circ \mathbf{h}_2$. Then for all vectors $\mathbf{g}\in \mathbb{R}^d$
\begin{equation}
    \mathbb{E}[{\langle \mathbf{g},\mathbf{v}_1 + \mathbf{v}_2 \rangle}^2] \ge C{\Vert \mathbf{g} \Vert}^2,
\end{equation}
where constant $C\in \mathbb{R}$ is a constant such that $C \le \frac{1}{4} ({\Vert \mathbf{h}_1 \Vert}^2 \cos^2{\theta_1} + {\Vert \mathbf{h}_2 \Vert}^2 \cos^2{\theta_2}) $, and $\theta_1$, $\theta_2$ are the angles between $\mathbf{g}$ and $\mathbf{h}_1$, $\mathbf{h}_2$, respectively.
\end{lemma}
See Appendix Sec. \ref{sec-appendix:proof_guided_line} for the proof. In Lemma \ref{lemma:informed-direction}, $\mathbf{h}_1$ and $\mathbf{h}_2$ represents the two incumbent directions, $\mathbf{\bar{p}}$ and $\mathbf{\bar{g}}$ in Eq. (\ref{eq:velocity}). From Lemma \ref{lemma:informed-direction}, we derive the following simple regret bound.
\begin{proposition} \label{prop:regret}
Let Assumptions \ref{assumption:rkhs} and \ref{assumption:gamma_t} hold. After $T$ iterations of \psbo with random directions that satisfy Lemma \ref{lemma:informed-direction}, the expected simple regret over the random vectors $\mathbf{r}_1$ and $\mathbf{r}_2$ is bounded as follows,
\begin{equation}
    \mathbb{E}[r_T] \le \mathcal{O}((d \log{T}/ T)^{1/2-\kappa} ).
\end{equation}
\end{proposition}
See Appendix Sec. \ref{sec-appendix:proof_simple_regret} for the proof. In the case of our implementation with Matern 5/2 kernel for the experiments, the simple regret is then bounded by $\mathbb{E}[r_T] \le \mathcal{O}\left((d \log{T}/ T)^{3/14} \right)$. 
Compared to LineBO \cite{kirschner2019linebo}, our simple regret is tighter bounded in terms of the input dimension $d$, while having a similar rate in terms of $T$.

\subsection{Global Convergence} \label{sec-theory:global}
As \psbo is a line-based optimization method, it inherits the global convergence property from LineBO (Theorem 1 of \citeauthor{kirschner2019linebo}, \citeyear{kirschner2019linebo}). 
However, because \psbo also incorporates the linear random subspace embedding technique, the global convergence property depends on whether or not the subspace $\mathcal{A}$ contains the global optimum $\mathbf{x}^*_{\mathcal{X}}$, i.e., there exists a point $\mathbf{x}^*_\mathcal{A} \in \mathcal{A}$ such that $\mathbf{S}^\intercal \mathbf{x}^*_\mathcal{A} = \mathbf{x}^*_{\mathcal{X}}$, where $\mathbf{S}$ is the embedding matrix.
Previous works analyze this property via \textit{the probability of the embedding to contain the global optimum} \cite{Nayebi2019hesbo, Letham2020Alebo, papenmeier2022baxus}. As \psbo employs BAxUS embedding, we use the following theorem (Theorem 1 of \citeauthor{papenmeier2022baxus}, \citeyear{papenmeier2022baxus}) to reason about this probability,
\begin{theorem} \label{theorem:global}
Let $\alpha=\lfloor d/d_\mathcal{A}\rfloor$, $\beta = \lceil d/d_\mathcal{A} \rceil$. Denote $\binom{a}{b}$ as the binomial coefficients. The probability of BAxUS embedding to contain the global optimum is,
\begin{equation} \label{eq:baxus-worst-case-prob}
    p^*=\frac{\sum_{i=0}^{d_e}\binom{d_\mathcal{A}(1+\alpha) - d}{i} \binom{d-d_\mathcal{A}\beta}{d_e-i} \alpha^i \beta^{d_e-i}}{\binom{d}{d_e}}.
\end{equation}
\end{theorem}
A useful result from Theorem \ref{theorem:global} is that the worst case success probability becomes one when $d_\mathcal{A}=d$, guaranteeing that the embedding in \psbo always contains the global optimum. This property, together with the simple regret bound in Proposition \ref{prop:regret}, concludes \psbo convergence property.

\section{Experiments}
In this section, we extensively evaluate our proposed method \psbo against a set of baselines on a set of benchmark problems, showing that \psbo outperforms state-of-the-art baselines. We additionally conduct ablation study to understand each component of the proposed method.

\subsection{Experimental Settings and Baselines} \label{sec-exp:baselines_settings}

We evaluate our proposed algorithm \textbf{\psbo} against a comprehensive set of baselines:
\textbf{Standard BO}, 
\textbf{LineBO} \cite{kirschner2019linebo}, 
\textbf{BAxUS} \cite{papenmeier2022baxus}, 
\textbf{HESBO} \cite{Nayebi2019hesbo}, 
\textbf{ALEBO} \cite{Letham2020Alebo}, 
\textbf{SAASBO} \cite{eriksson21saasbo}, 
\textbf{TuRBO} \cite{Eriksson2019TuRBO}, 
\textbf{CMA-BO} \cite{ngo2024cmabo}, 
\textbf{RDUCB} \cite{ziomek2023rducb},
\textbf{CMA-ES} \cite{Hansen2001CMAES} and 
\textbf{PSO} \cite{kennedy1995pso}.

For ALEBO and HESBO, we run each method with different target dimensions $d_\mathcal{A}=10$ and $d_\mathcal{A}=20$. For LineBO, we compare to CoordinateLineBO version, as it was identified as the best variant proposed by \citeauthor{kirschner2019linebo}, \citeyear{kirschner2019linebo}.
We run all methods 10 times with different random seeds and report the mean and standard error. Most methods are run with 1000 iterations on all problems, except for SAASBO (100 iterations), ALEBO (500 iterations) and RDUCB (2 days running time) due to high computational time and memory required. The settings of \psbo are as follows. For the incumbent-guided direction lines, we use a common configuration as with PSO, $m=20$, $w=0.729$, $c_1=c_2=2.05w$ \cite{shi1998wPSO, regis2014PSO_RBF, li2020PSO_FastSurrogate}. 
For the surrogate model, we use a GP with Mat\'{e}rn 5/2 kernel function. For acquisition function, we use Thompson Sampling and employ NSGA-II algorithm \cite{deb2002nsgaii} as the MO optimizer for Eq. (\ref{eq:moo-acq}). See Appendix Secs. \ref{sec-appendix:psbo-settings} and \ref{sec-appendix:baselines-settings} for the detailed implementation of \psbo and the baselines. 

\begin{figure*} [t]
  \centering
  \includegraphics[width=0.97\textwidth]{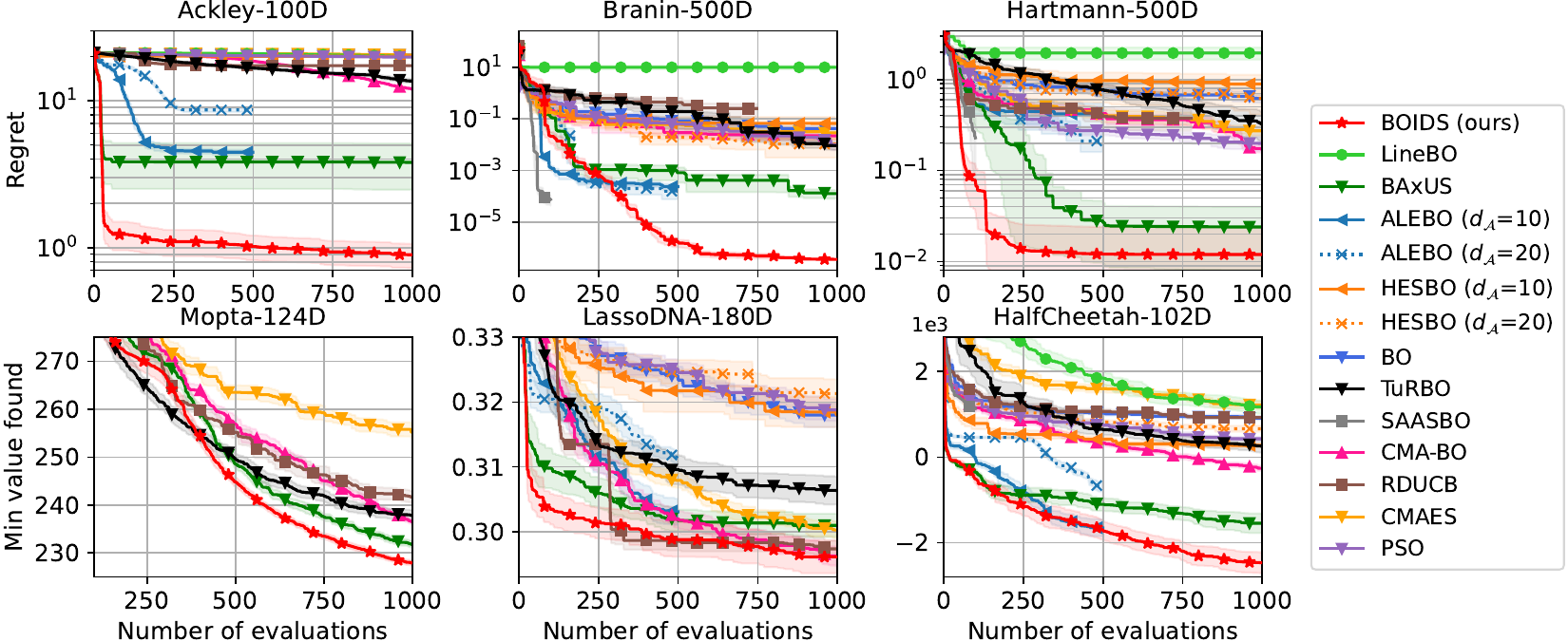}
\caption{Comparison of our proposed method, \psbo, against the state-of-the-art baselines on six minimization problems. Note that some methods (ALEBO, SAASBO and RDUCB) only have limited iterations due to the prohibitively high computational cost and memory required. Overall, \psbo outperforms all baselines significantly on most problems.} \label{fig:main-results}
\end{figure*}

\subsection{Benchmark Problems}
We conduct experiments on 3 synthetic and 3 real-world benchmark problems. The dimensions of these problems range from $100$ to $500$. For synthetic problems, we use \textit{Ackley-100D}, \textit{Branin-500D} and \textit{Hartmann-500D}, which are widely-used in BO research works \cite{Eriksson2019TuRBO, song2022mctsvs, ngo2024cmabo}.
For real-world benchmark problems, we use \textit{Mopta-124D}, \textit{LassoDNA-180D} and \textit{HalfCheetah-102D}.
Mopta-124D is a vehicle design task aiming to minimize the vehicle's weight, with code implementation from \cite{eriksson21saasbo}.
LassoDNA-180D is a hyperparameter tuning task involving a microbiology DNA dataset, with code implementation from \cite{vsehic2022lassobench}.
HalfCheetah-102D is a reinforcement learning task aiming to maximize cumulative reward, with code implementation from \cite{song2022mctsvs}. 
Details on each benchmark problem are in Appendix Sec. \ref{sec-appendix:benchmark_problems}.

\subsection{Comparison with Baselines}
Fig. \ref{fig:main-results} shows the performance of all baselines across all six benchmark problems. On synthetic problems (Ackley-100D, Branin-500D and Hartmann-500D), \psbo significantly outperforms all baselines. BAxUS converges to less optimal values. SAASBO, despite being a strong baseline, especially on Branin-500D, is prohibitively expensive and is only suitable for very limited budgets (up to 64 GB Memory for 100 iterations). For real-world benchmark problems, \psbo consistently outperforms all baselines. On Mopta-124D, \psbo shows a slower start compared to TuRBO due to its initial focus on low-dimensional subspaces, a behavior observed in subspace embedding techniques including BAxUS. However, \psbo quickly finds more optimal solutions as iterations progress. On LassoDNA-180D, even though both RDUCB and CMA-BO find similarly optimal solutions as \psbo, \psbo converges much faster. On HalfCheetah-102D, ALEBO ($d_\mathcal{A}=10$) is competitive to \psbo but requires careful tuning of the target dimension, with ALEBO ($d_\mathcal{A}=20$) showing significantly worse performance. Additionally, ALEBO is also prohibitively expensive to run (up to 40 hours for 500 iterations for each repeat). 
Overall, \psbo demonstrates superior performance across these benchmark problems, indicating a robust and efficient solution for high-dimensional optimization problems.

\subsection{Ablation Study}
We conduct an ablation study to investigate the effectiveness of each component in \psbo by alternatively removing each of the four components of \psbo (as described in Sec. \ref{sec-method:overall-algorithm}) and compare the four corresponding versions with the original \psbo. First, in version ``\textbf{W/o Guided Direction}'', we replace the incumbent-guided line component with uniformly random guiding lines $\mathcal{L}(\mathbf{\hat{x}}, \mathbf{v})$, where $\mathbf{v} \sim \mathcal{U}([l-u,u-l]^d)$. Second, in version ``\textbf{W/o Line Select}'', we remove the MAB-based line selection component and instead select the lines randomly in each iteration, $\mathcal{L}_i$ where $i=\mathcal{U}([1\dots m])$. Third, in version ``\textbf{W/o Line Opt}'', we omit the line-based optimization component, and instead optimize the acquisition function over the entire input space $\mathcal{X}$, i.e., $\mathbf{x}^{(t+1)}=\argmax_{\mathbf{x}\in\mathcal{X}}{\alpha(\mathbf{x})}$. Finally, in version ``\textbf{W/o Embedding}'', we exclude the subspace embedding component and perform optimization directly in the input dimension $d_\mathcal{A}=d$ and $\mathbf{S}=\mathbf{I}_d$. All other settings in the four versions are kept similar to the original \psbo.

We run these versions on two benchmark problems - LassoDNA-180D and HalfCheetah-102D - with 10 repeats, and compare them against the original method. We report the mean and standard error of the results in Fig. \ref{fig:abalation-results}. Overall, the results show that the original \psbo consistently shows the best performance. Among the components, the line selection component via MAB (Sec. \ref{sec-method:line-direction-selection}) seems to be the least important, as removing it does not significantly degrade the performance. Conversely, the incumbent-guided line component (Sec. \ref{sec-method:incumbent-line}) and line-based optimization component (Sec. \ref{sec-method:line-based-optimization}) are crucial due to the significant performance degradation in the related versions. The subspace embedding component (Sec. \ref{sec-method:subspace-embedding}) greatly affects the performance in HalfCheetah-102D, yet has less impact on LassoDNA-180D problem. 

\begin{figure} [h]
  \centering
  \begin{subfigure}[b]{\columnwidth}
   \centering
   \includegraphics[width=\columnwidth]{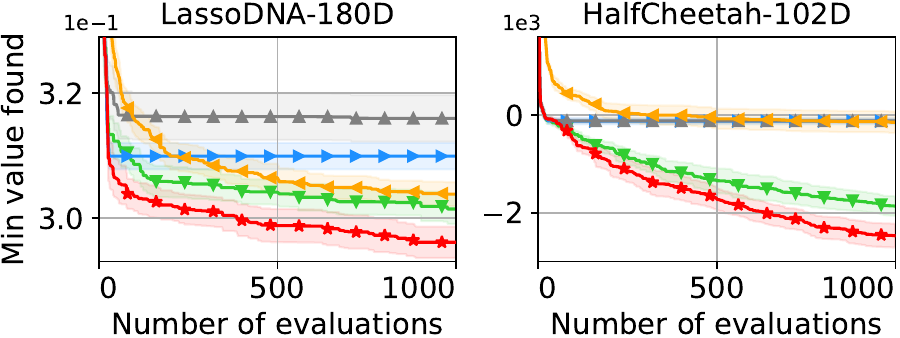}
  \end{subfigure}
  \begin{subfigure}[b]{\columnwidth}
   \centering
   \includegraphics[width=\columnwidth]{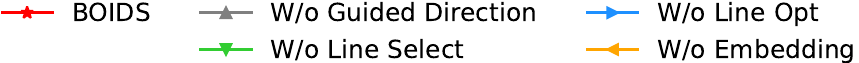}
  \end{subfigure}
\caption{Impact of \psbo components on performance when each is alternatively removed. The \textit{incumbent-guided line} component (Sec. \ref{sec-method:incumbent-line}) and the tailored \textit{line-based optimization} (Sec. \ref{sec-method:line-based-optimization}) have the most significant influence.} \label{fig:abalation-results}
\end{figure}

\section{Conclusion}
In this paper, we address the high-dimensional optimization problems for expensive black-box objective functions. We propose \psbo, a novel line-based BO algorithm using incumbent-guided direction lines. To further improve the efficiency, we select the most optimal line via a MAB approach. To further boost the performance, we employ subspace embedding technique. We theoretically analyze \psbo's convergence property and empirically demonstrate that \psbo outperforms state-of-the-art baselines on a variety of synthetic and real-world tasks.

\section*{Acknowledgments}
% This research is supported by Australian Research Council Discovery Project DP220103044. 
The first author (L.N.) would like to thank the School of Computing Technologies, RMIT University, Australia for providing computing resources on this project. Additionally, this project was undertaken with the assistance of computing resources from RACE (RMIT AWS Cloud Supercomputing Hub).

\bibliography{aaai25}

\clearpage % for submission only
\appendix
\section{Appendix}
\subsection{Details on the Incumbents Computation} \label{sec-appendix:incumbents}
In this section, we provide more details on the computation of the local and global incumbents. Note that the local and global incumbents are identical to the concepts of local and global best in PSO \cite{kennedy1995pso}. The main difference is that while PSO updates all $m$ data points in each iteration, \psbo sequentially select one data point to update in each iteration.

To compute the local incumbent $\mathbf{p}_i^{(t)}$, we keep track of the history data of $i$-th data point $\mathbf{x}_i$. As each BOIDS iteration selects one associated incumbent-guided line $\mathcal{L}_i$ to suggest the next data point for observation, this will update $\mathbf{x}_i$ to new location. The history data of $\mathbf{x}_i$ is collected by recording these updates made to all $\mathbf{x}_i$ during optimization. Then, $\mathbf{p}_i^{(t)}$ is the best location of $\mathbf{x}_i$ within its history (up to iteration $t$), in terms of function evaluation. In formal terms, let us denote the history location of $i$-data point as $\{\mathbf{x}_i^{(k)}\}|_{k=1}^{K_i^{(t)}}$, where $K_i^{(t)}$ represents the update counter of $i$-th data point up to iteration $t$. Then $\mathbf{p}^{(t)}_i=\argmin_{k=1,\dots,K_i^{(t)}}{f(\mathbf{x}_i^{(k)})}$. Note that the update counter $K_i$ of $i$-th data point is generally not equal to those of the remaining $m-1$ data points, as BOIDS may update different data point in each iteration. 

The global incumbent (global best) $\mathbf{g}^{(t)}$ is the best solution found across all $m$ data points up to iteration $t$. This is simply the best observed data point found so far up to iteration $t$, i.e., $\mathbf{g}^{(t)}=\argmin_{j=1,\dots,t}{f(\mathbf{x}^{(j)})}$.

\subsection{Proof for Lemma \ref{lemma:informed-direction}} \label{sec-appendix:proof_guided_line}
We show the proof for Lemma \ref{lemma:informed-direction} concerning the incumbent-guided direction, which is a key step in deriving simple regret of \psbo as outlined in Proposition \ref{prop:regret}.  
% We first need to prove Lemma \ref{lemma:informed-direction}.
\begin{proof}[Proof of Lemma \ref{lemma:informed-direction}]
Ultimately, we aim to compute $\mathbb{E}[{\langle \mathbf{g}, \mathbf{v}_1 + \mathbf{v}_2 \rangle}^2]$, where $\mathbf{g} \in \mathbb{R}^d$. Before doing so, we separate the summation and derive $\mathbb{E}[{\langle \mathbf{g}, \mathbf{v}\rangle}^2]$, where $\mathbf{v} = \mathbf{r} \circ \mathbf{h}$ and $\mathbf{r} \sim \mathcal{U}([0,1]^d)$.
\begin{equation*}
\begin{split}
\mathbb{E}[{\langle \mathbf{g}, \mathbf{v}\rangle}^2] & = \mathbb{E}\left[ {\left(\sum_{i=1}^d {g_i r_i h_i} \right)}^2\right] \\
 & = \mathbb{E}\left[ \sum_{i=1}^{d}{(g_i r_i h_i)^2} + 2\mathop{\sum\sum}_{i\neq j}{g_i g_j r_i r_j h_i h_j}\right] \\
 & = \sum{g_i^2 h_i^2 \mathbb{E}\left[r_i^2\right]} + 2\mathop{\sum\sum}_{i\neq j}{g_i g_j h_i h_j \mathbb{E}\left[r_i r_j \right]}. \\
 % & = \mathbb{E}\left[ \right] \\
\end{split}
\end{equation*}
As $\mathbf{r}$ is uniformly random vector $\mathcal{U}([0,1]^d)$, $\mathbb{E}\left[r_i^2\right]=1/3$ and $\mathbb{E}\left[r_i r_j\right]=\mathbb{E}\left[r_i\right]  \mathbb{E}\left[r_j\right]=1/4$. By rearranging the terms, the expectation leads to,
\begin{equation} \label{eq-appendix:1}
\begin{split}
\mathbb{E}[{\langle \mathbf{g}, \mathbf{v}\rangle}^2] & = \frac{1}{12}{\Vert \mathbf{g} \circ \mathbf{h}\Vert}^2 + \frac{1}{4} {\langle {\mathbf{g}, \mathbf{h} \rangle}}^2 \\
& \geq {\Vert \mathbf{g} \Vert}^2 \frac{{\Vert \mathbf{h} \Vert}^2 \cos{\theta}^2}{4},
\end{split}
\end{equation}
where $\theta$ is the angle between $\mathbf{g}$ and $\mathbf{h}$. Now, we proceed to compute $\mathcal{E} = \mathbb{E}[{\langle \mathbf{g}, \mathbf{v}_1 + \mathbf{v}_2\rangle}^2]$. By linearity of dot product, we have
\begin{equation*}
\begin{split}
\mathcal{E} & = \mathbb{E}\left[ {\left(\langle \mathbf{g}, \mathbf{v}_1 \rangle  + \langle \mathbf{g},  \mathbf{v}_2\rangle\right)}^2 \right]  \\
& = \mathbb{E}\left[ {\langle \mathbf{g}, \mathbf{v}_1 \rangle}^2 \right]  + \mathbb{E}\left[{\langle \mathbf{g}, \mathbf{v}_2\rangle}^2\right] + 2\mathbb{E}\left[{\langle \mathbf{g}, \mathbf{v}_1\rangle} {\langle \mathbf{g}, \mathbf{v}_2\rangle} \right]  \\
% & = \mathbb{E}\left[ {\langle \mathbf{g}, \mathbf{v}_1 \rangle}^2 \right]  + \mathbb{E}\left[{\langle \mathbf{g}, \mathbf{v}_2\rangle}^2\right]. \\
\end{split}
\end{equation*}
The last equality is due to linearity of expectation. In the results, the third expectation is $\mathbb{E}\left[{\langle \mathbf{g}, \mathbf{v}_1\rangle} {\langle \mathbf{g}, \mathbf{v}_2\rangle} \right] = \mathbb{E}\left[ (\mathbf{g}^\intercal \mathbf{v}_1) (\mathbf{g}^\intercal \mathbf{v}_2) \right] = \mathbb{E}\left[ \mathbf{g}^\intercal (\mathbf{v}_1 \mathbf{v}_2^\intercal) \mathbf{g} \right] = 0$ as $\mathbb{E}[\mathbf{v}_1 \mathbf{v}_2^\intercal] = 0$ because $\mathbf{v}_1$ and $\mathbf{v}_2$ are independent. Hence, we have,
\begin{equation*}
\begin{split}
\mathbb{E}[{\langle \mathbf{g}, \mathbf{v}_1 + \mathbf{v}_2\rangle}^2] & = \mathbb{E}\left[ {\langle \mathbf{g}, \mathbf{v}_1 \rangle}^2 \right]  + \mathbb{E}\left[{\langle \mathbf{g}, \mathbf{v}_2\rangle}^2\right]. \\
\end{split}
\end{equation*}
Combination of this result with Eq. (\ref{eq-appendix:1}) concludes the proof.
\end{proof}

\subsection{Proof for Proposition \ref{prop:regret}} \label{sec-appendix:proof_simple_regret}
Given Lemma \ref{lemma:informed-direction}, we follow \citeauthor{srinivas2009gpucb} and \citeauthor{kirschner2019linebo} to derive the simple regret of \psbo. Similar to LineBO \cite{kirschner2019linebo}, we impose additional assumptions on the smoothness and convexity of the objective function.

% \hh{Do we mention this in the main paper? And are these assumptions used in the LineBO paper?} \lng{We do not directly mention this in the main paper. These are used in LineBO paper in the appendix as well} \hh{then need to cite the LineBO paper and say this is similar to that paper} \lng{updated}.

\begin{definition}[$\alpha$-convexity] We call a differentiable function $f$ is $\alpha$-convex if there exist $\alpha>0$ such that for all $\mathbf{x},\mathbf{h}\in\mathbb{R}^d$,
\begin{equation}
    \langle \nabla f(\mathbf{x}), \mathbf{h} \rangle + \frac{\alpha}{2}{\Vert \mathbf{h}\Vert}^2 \leq  f(\mathbf{x}+\mathbf{h}) - f(\mathbf{x}).
\end{equation}
\end{definition}
\begin{definition}[$\beta$-smoothness] We call a differentiable function $f$ is $\beta$-smooth if there exist $\beta>0$ such that for all $\mathbf{x},\mathbf{h}\in\mathbb{R}^d$,
\begin{equation}
    f(\mathbf{x}+\mathbf{h}) - f(\mathbf{x}) \leq \langle \nabla f(\mathbf{x}), \mathbf{h} \rangle + \frac{\beta}{2}{\Vert \mathbf{h}\Vert}^2.
\end{equation}
\end{definition}

\begin{proof} [Proof of Proposition \ref{prop:regret}]
Assume that the function $f$ if $\alpha$-convex and $\beta$-smooth. Denote $\hat{\mathbf{x}}^{(t)}$ as the proposed data from running Algorithm \ref{alg:ps-bo-alg} up to iteration $t$ with incumbent-guided direction $\mathcal{L}$ following Lemma \ref{lemma:informed-direction}. Assume that $\hat{\mathbf{x}}^{(t)} \in \mathcal{L}$. Denote a local optimum of $f$ as $\mathbf{x}_*$. Denote the solution from \textit{exact} line search on incumbent-guided line as $\mathbf{x}_*^{(t+1)} = \argmin_{\mathbf{x} \in \mathcal{L}} f(\mathbf{x})$. Then, we follow \citeauthor{kirschner2019linebo} (Lemma 4) to yield the following result,
\begin{equation} \label{eq-appendix:2}
    \mathbb{E}\left[ f(\mathbf{x}_*^{(t+1)}) - f(\mathbf{x}_*) \right] \le \left(1 - \frac{\alpha C}{2\beta} \right)(f(\hat{\mathbf{x}}^{(t)}) - f(\mathbf{x}_*)),
\end{equation}
where the expectation is over the random $\mathbf{r}_1$ and $\mathbf{r}_2$. Denote the accuracy loss when finding the proposed solution $\hat{\mathbf{x}}^{(t+1)}$ instead of the exact one $\mathbf{x}_*^{(t+1)}$ as $\varepsilon$, and $\gamma = \frac{\alpha C}{2\beta}$, we compute the expected regret of solution $\hat{\mathbf{x}}^{(t+1)}$ as,
\begin{equation*}
\begin{split}
 \mathbb{E}\left[ f(\hat{\mathbf{x}}^{(t+1)}) - f(\mathbf{x}_*)\right] &\leq \mathbb{E}\left[ f(\mathbf{x}_*^{(t+1)}) - f(\mathbf{x}_*)\right] + \varepsilon. \\
\end{split}
\end{equation*}
Applying Eq. (\ref{eq-appendix:2}) yields,
\begin{equation*}
\begin{split}
 \mathbb{E}\left[ f(\hat{\mathbf{x}}^{(t+1)})  - f(\mathbf{x}_*)\right] &\leq (1-\gamma)(f(\hat{\mathbf{x}}^{(t)}) - f(\mathbf{x}_*)) + \varepsilon,
\end{split}
\end{equation*}
This result shows the relationship between solutions from \psbo in two consecutive iterations. Applying this result recursively up to iteration $T$, we yield the expected regret as follows,
\begin{equation*}
\begin{split}
 \mathbb{E}\left[ f(\hat{\mathbf{x}}^{(T)}) - f(\mathbf{x}_*)\right] &\leq \varepsilon\sum_{t=1}^{T}{(1-\gamma)^t} \\
 & + (1-\gamma)^T(f(\hat{\mathbf{x}}^{(0)})-f(\mathbf{x}_*)).\\
 & \leq \frac{\varepsilon}{\gamma} + (1-\gamma)^T(f(\hat{\mathbf{x}}^{(0)})-f(\mathbf{x}_*)).
\end{split}
\end{equation*}
Given the kernel is a one-dimensional kernel following Assumption \ref{assumption:gamma_t}, we follow \citeauthor{kirschner2019linebo} to set the accuracy $\varepsilon=(\frac{d\log{T}}{2T})^{(1-2\kappa)/2}$ and rewrite the expected simple regret as,
\begin{equation*}
 \mathbb{E}\left[ r_T\right] \leq \mathcal{O}\left( (d\log{T}/T)^{(1-2\kappa)/2} \right).
\end{equation*}
The simple regret is derived by assuming $f^*=f(\mathbf{x}_*)$, hence $r_T = \min{y^{(t)}}|_{i=1}^T - f^* \leq f(\hat{\mathbf{x}}^{(T)}) - f(\mathbf{x}_*)$. This concludes the proof.
\end{proof}

% Note that the proof relies on a strict assumption that the solution of Eq. (\ref{eq:moo-acq-select}) is on (or very close to) the incumbent-guided lines, $\hat{\mathbf{x}}^{(t)} \in \mathcal{L}$. Therefore, this proof can be used as a sketch for deeper theoretical analysis when relaxing this assumption.

\subsection*{Outline for Proof for Theorem \ref{theorem:global}}
% \hh{How about Theorem 1? Need to say something about it here.} \lng{I outlined the main idea here. Should we separate this paragraph into new subsection? We will not be able to link to the main paper} \hh{we can still create a subsection, and don't have to link to the main paper - but will the section name in the main paper change?} \lng{no we don't change the main paper, I think we just need to create a subsection named ``(Outline for) Proof for Theorem 1'' for this paragraph} \hh{I meant if we have a new subsection, will the later subsection number change (e.g., A.3 becomes A.4, etc)? And will it affect the text in the main paper?} \lng{ah I see. Then we can use subsection*, it will create a non-numbered section? I just gave it a try.} 

For the proof of Theorem \ref{theorem:global}, we refer readers to Theorem 1 of \citeauthor{papenmeier2022baxus}, \citeyear{papenmeier2022baxus}. The general idea is to compute the ratio between number of possible optima-preserving assignments - the number of possible ways to distribute $d_e$ effective dimensions into $d_\mathcal{A}$ target dimensions - to the total number of assignments - the number of possible ways to distribute $d_e$ effective dimensions into all $d$ input dimensions.

\subsection{Adaptive Line Selection} \label{sec-appendix:line-direction-selection}
We show in Alg. \ref{alg:line-select} the pseudo code for adaptively selecting the most optimal incumbent-guided line $\mathcal{L}_i$ among the set of $m$ lines $\mathbb{L}=\{\mathcal{L}(\hat{\mathbf{x}}_i,\mathbf{v}_i) \}|_{i=1}^m$. 

\begin{algorithm}[ht] 
   \caption{\texttt{LINE-SELECT}: Line Selection Algorithm}
   \label{alg:line-select}
\begin{algorithmic}[1]
   \State {\bfseries Input:} Current $\mathcal{GP}$, set of $m$ lines $\mathbb{L}=\{\mathcal{L}_i\}|_{i=1}^m$
   \State {\bfseries Output:} The most optimal line $\mathcal{L}_{i^*}$
   \State Sample a posterior $g \sim \mathcal{GP}$
   \For{$\mathcal{L}_i \in \mathbb{L}$}
     \State Sample candidate data points $\mathcal{X}_c$ on line $\mathcal{L}_i$
     \State Find best TS reward $r_i = \max_{{\mathbf{x}_c} \in \mathcal{X}_c}{g(\mathbf{x}_c)}$ 
     \State Append $r_i$ to $\mathcal{R}$
   \EndFor
   \State Find best line $i^* = \argmax_{r_i \in \mathcal{R}}{\mathcal{R}}$
   \State Return $\mathcal{L}_{i^*}$
\end{algorithmic}
\end{algorithm}

\subsection{Incumbent-guided Line-based Optimization} \label{sec-appendix:line-based-optimization}
We show in Alg. \ref{alg:line-based-opt} the pseudo code for the line-based optimization procedure tailored for the incumbent-guided direction lines.

\begin{algorithm}[ht] 
   \caption{\texttt{LINE-OPT}: Incumbent-guided Line Optimization}
   \label{alg:line-based-opt}
\begin{algorithmic}[1]
   \State {\bfseries Input:} The incumbent-guided direction line $\mathcal{L}(\mathbf{\hat{x}}, \mathbf{v})$.
   \State {\bfseries Output:} The next data point for observation.
   \State Sample candidate data points $\mathbf{x}_c \in \mathcal{L}(\mathbf{\hat{x}}, \mathbf{v})$
   \State Solve Eq. (\ref{eq:moo-acq}) for Pareto set $\mathcal{P}$ given initial $\mathbf{x}_c$ 
   \State Solve Eq. (\ref{eq:moo-acq-select}) for the next BO data points $\mathbf{x}^* \in \mathcal{P}$ 
   \State Return $\mathbf{x}^*$
\end{algorithmic}
\end{algorithm}

\subsection{Detailed Implementation of \psbo} \label{sec-appendix:psbo-settings}
For the incumbent-guided lines, we use a common configuration as with PSO, $m=20$, $w=0.729$, $c_1=c_2=2.05w$ \cite{shi1998wPSO, regis2014PSO_RBF, li2020PSO_FastSurrogate}. 

For the surrogate model, we follow BAxUS to implement the GP via \texttt{GpyTorch} \cite{gardner2018gpytorch} and \texttt{BoTorch} \cite{balandat2020botorch}. Specifically, we use a Mat\'{e}rn 5/2 kernel with the ARD length-scales in the interval $[0.005, 10.0]$ and signal variance in the interval $[0.05, 20.0]$. The Gaussian likelihood is modelled with standard homoskedastic noise in the interval $[0.0005, 0.2]$.

% \hh{Are these standard choices? Are these values used in other papers? And are these values similar to the settings of the baselines?} \lng{Yes, they are similar to BAxUS setting for GP} \hh{then need to mention that} \lng{updated}

For the acquisition function, we use NSGA-II to solve the multi-objective optimization problems. We use the NSGA-II implementation from \texttt{pymoo} \cite{pymoo} with default settings: population size of 100, tournament selection, simulated binary crossover and polynomial mutation. We terminate NSGA-II runs after 100 generations reached.

For the subspace embedding, we follow recommended settings in BAxUS, including the bin size $b=3$, budget until input dimension $1000$. To control the budget for each target dimension $T_{d_\mathcal{A}}$, we also employ the success and failure counter as in BAxUS. Note that because we do not employ TR, we employ a termination factor $K$ as the threshold for terminating the search in target space $\mathcal{A}$. Specifically, if a better solution is found in an iteration, we increase the success counter $c_s$, while increasing the failure counter $c_f$ otherwise (we also reset the other counter to zero). If $c_f$ exceeds the failure threshold $\tau_\text{fail}$, which is computed from $d_\mathcal{A}$, we increase $K$ by 1. If $c_s$ exceeds the success threshold $\tau_\text{succ} = 3$, we decrease $K$ by 1. $K$ is maintained such that $K\ge K_{\min}$, and if $K$ exceeds a maximum threshold $K_{\max}$, we terminate the search in target space $\mathcal{A}$. The termination factor $K$ can be considered as the counterpart for the TR side length in TurBO \cite{Eriksson2019TuRBO} and BAxUS \cite{papenmeier2022baxus}. In the experiments, we use an equivalent bound for $K$, such that $K_\text{init}=1$, $K_{\min} = 0$ and $K_{\max}=7$.

We implement \psbo in Python (version 3.10). We use Miniconda (version 23.3.1) to install Python packages. We provide in the code supplementary a \texttt{.yml} file to install the required packages for running our proposed methods.

\subsection{Implementations of Baselines} \label{sec-appendix:baselines-settings}
We run all baselines in Python using their open-sourced code. When evaluating the benchmark functions, we use the same implementation for fair comparison.

\paragraph{LineBO \cite{kirschner2019linebo}.} We compare with \textbf{CoordinateLineBO} version, which is overall the best compared to RandomLineBO and DescentLineBO. We keep all of the default settings from the paper and the code implementation, including RBF kernel and UCB acquisition function. We use the open-sourced implementation at \url{https://github.com/kirschnj/LineBO}.

\paragraph{BAxUS \cite{papenmeier2022baxus}.} We use the same settings as the embedding management with our proposed methods, including bin size $b=3$, budget until input dimension is equal to the maximum evaluation. The budget management is kept by default. The trust region management is kept by default, which is similar to TuRBO (see TuRBO settings below). We use the open-sourced implementation at \url{https://github.com/LeoIV/BAxUS}.

\paragraph{HESBO \cite{Nayebi2019hesbo}.} We compare with HESBO given two different settings of target dimension $d_\mathcal{A}=10$ and $d_\mathcal{A}=20$. We use the open-sourced implementation at \url{https://github.com/aminnayebi/HesBO}.

\paragraph{ALEBO \cite{Letham2020Alebo}.} We compare with ALEBO given two different settings of target dimension $d_\mathcal{A}=10$ and $d_\mathcal{A}=20$. We use the implementation from \texttt{ax-platform} at \url{https://github.com/martinjankowiak/saasbo}.

\paragraph{SAASBO \cite{eriksson21saasbo}.} We use default settings from the paper and implementation. We use the open-sourced implementation at \url{https://github.com/uber-research/TuRBO}.

\paragraph{TuRBO \cite{Eriksson2019TuRBO}.} We use default settings from the paper and implementation, including minimum and maximum trust region size of $\{2^{-7}; 1.6\}$, trust region split ratio of $2$. We use the open-sourced implementation at \url{https://github.com/uber-research/TuRBO}.

\paragraph{CMA-BO \cite{ngo2024cmabo}.} We use default settings from the paper and code implementation, including population size $\lambda$, initial variance $\sigma_0=0.3(u-l)$ and BO optimizer. We use the open-sourced implementation at \url{https://github.com/LamNgo1/cma-meta-algorithm}.

\paragraph{RDUCB \cite{ziomek2023rducb}.} We use the default settings from the paper and code implementation. We use the open-sourced implementation at \url{https://github.com/huawei-noah/HEBO/tree/master/RDUCB}.

\paragraph{CMA-ES \cite{Hansen2001CMAES}.} We use default settings from the paper and code implementation, including initial variance $\sigma_0=0.3(u-l)$. We use the implementation from \texttt{pycma} package at \url{https://github.com/CMA-ES/pycma}.

\paragraph{PSO \cite{kennedy1995pso}.} We use the swarm settings similar to our proposed method, including population size $m=20$, $w=0.729$, $c_1=c_2=2.05w$. For topology setting, we use star topology, which is similar to the mechanism used in \psbo. We use the implementation from \texttt{pyswarms} package at \url{https://github.com/ljvmiranda921/pyswarms}.

\subsection{Details of Benchmark Problems} \label{sec-appendix:benchmark_problems}

\paragraph{Synthetic Problems.} The three synthetic benchmark problems - \textit{Ackley-100D}, \textit{Branin-500D} and \textit{Hartmann-500D} - are derived from common test problems widely used in optimization\footnote{\url{https://www.sfu.ca/~ssurjano/index.html}}. The \textit{Ackley-100D} problem is based on the Ackley function, with 100 variables affecting the function values. We consider the standard hypercube domain $[-32.768, 32.768]^{100}$. The \textit{Branin-500D} and \textit{Hartmann-500D} problems are created by augmenting Branin-2D and Hartmann-6D problems with dummy variables that do not affect function values. The \textit{Hartmann-500D} is evaluated on a hypercube $[0,1]^{500}$, while \textit{Branin-500D} is evaluated on a domain consisting of $[-5,10]$ and $[0, 15]$ for the 2 effective dimensions, and $[0,1]$ for the remaining 498 dummy dimensions. These synthetic functions have been widely-used in BO research \cite{wang2016rembo, Eriksson2019TuRBO, Letham2020Alebo, eriksson21saasbo, papenmeier2022baxus, ngo2024cmabo, ziomek2023rducb, Letham2020Alebo}.

\paragraph{LassoDNA-180D.} This problem is a hyper-parameter optimization problem for Weighted LASSO (Least Absolute Shrinkage and Selection Operator) regression. The goal is to tune a set of hyper-parameters to balance the least-square estimation and the penalty term for sparsity. Specifically, \textit{LassoDNA-180D} is a 180D hyper-parameter optimization problem that uses a DNA dataset from microbiology problem. We use the implementation from \texttt{LassoBench} package \cite{vsehic2022lassobench}. This package has also been used in \cite{papenmeier2022baxus,ziomek2023rducb,ngo2024cmabo}.

\paragraph{HalfCheetah-102D.} This problem is a reinforcement learning (RL) problem from the Mujoco locomotion tasks implemented in the \texttt{gym} package \cite{brockman2016openaigym}. The goal is to maximize the cumulative reward chosen by a linear policy, described by a high-dimensional matrix. Specifically, \textit{HalfCheetah-102D} is is a 102D reinforcement learning problem that is created from the Half-Cheetah-v4 environment from Mujoco tasks. We use the implementation from \citeauthor{song2022mctsvs}, \citeyear{song2022mctsvs}. These RL tasks have also been used in many BO works \cite{Wang2020LAMCTS, papenmeier2022baxus, ngo2024cmabo}.

\paragraph{Mopta-124D.} This vehicle design problem involves searching for a set of 124 variables to minimize the vehicle's weight. We follow \citeauthor{eriksson21saasbo}, \citeyear{eriksson21saasbo} to relax the 68 performance constraints into soft constraints, thereby, minimizing a scalar goal. We use \textit{Mopta-124D} implementation from \cite{eriksson21saasbo}. This problem has also been used in \cite{eriksson2021scbo, papenmeier2022baxus}.

\subsection{Computing Infrastructure}
We run experiments on a computing server with a Dual CPU of type AMD EPYC 7662 (total of 128 Threads, 256 CPUs). Each experiments are allocated 8 CPUs and 64GB Memory. 
The server is installed with Operating System Ubuntu 20.04.3 LTS.

\subsection{Runtime}
Table \ref{table:runtimes} reports the average runtime (in seconds) per iteration of \psbo and other baselines used in the main papers. 

\begin{table*}[htbp]
\centering
\begin{tabular}{|l|cccccc|}
\hline
\makecell[l]{Average runtime\\per iteration (s)} & \makecell[r]{Ackley\\-100D} & \makecell[r]{Branin\\-500D} & \makecell[c]{HalfCheetah\\-102D} & \makecell[c]{Mopta\\-124D} & \makecell[c]{LassoDNA\\-180D} & \makecell[c]{Hartmann\\-500D} \\
\hline
BOIDS & 11.90 & 23.10 & 12.41 & 12.88 & 15.16 & 23.04 \\
\hline
BO & 3.56 & 1.73 & 3.85 & 3.93 & 3.89 & 4.31 \\
\hline
BAxUS & 6.97 & 12.80 & 12.26 & 8.56 & 10.53 & 12.82 \\
\hline
LineBO & 0.04 & 0.06 & 0.48 & 0.34 & 0.73 & 0.06 \\
\hline
TuRBO & 3.88 & 2.97 & 3.57 & 3.92 & 3.99 & 4.42 \\
\hline
CMA-BO & 4.19 & 6.73 & 4.44 & 4.39 & 4.89 & 7.70 \\
\hline
SAASBO & 46.16 & 339.42 & 51.31 & 63.77 & 92.37 & 310.97 \\
\hline
HESBO ($d_e=10$) & 0.42 & 0.48 & 0.67 & 0.56 & 0.69 & 0.46 \\
\hline
HESBO ($d_e=20$) & 0.47 & 0.48 & 0.68 & 0.58 & 0.72 & 0.46 \\
\hline
ALEBO ($d_e=10$) & 36.60 & 65.44 & 74.33 & 70.43 & 89.80 & 71.24 \\
\hline
ALEBO ($d_e=20$) & 149.06 & 189.64 & 278.27 & 133.86 & 268.93 & 202.17 \\
\hline
RDUCB & 111.51 & 227.42 & 86.18 & 110.21 & 33.57 & 155.70 \\
\hline
CMA-ES & 3.4E-4 & 1.0E-3 & 0.06 & 0.05 & 0.08 & 8.7E-4 \\
\hline
PSO & 0.6E-4 & 0.8E-4 & 3.5E-3 & 0.01 & 3.0E-4 & 0.9E-4 \\
\hline
\end{tabular}
\caption{Runtime (in seconds) of \psbo and baselines. \psbo has higher runtime compared to BO, LineBO and BAxUS, yet much more feasible compared to ALEBO, SAASBO and RDUCB.}
\label{table:runtimes}
\end{table*}

\end{document}